\documentclass[twoside]{article}

\usepackage[english]{babel}
\usepackage[utf8x]{inputenc}
\usepackage[T1]{fontenc}
\usepackage[a4paper,top=3cm,bottom=2cm,left=3cm,right=3cm,marginparwidth=1.75cm]{geometry}

\usepackage{amsthm}
\usepackage{multirow}
\usepackage{caption}
\usepackage{amsmath,amssymb,amsfonts}
\usepackage{algorithm}
\usepackage[noend]{algpseudocode}
\usepackage{wrapfig}
\makeatletter
\def\BState{\State\hskip-\ALG@thistlm}
\makeatother
\usepackage{graphicx}
\usepackage{subcaption}
\usepackage[colorinlistoftodos]{todonotes}
\usepackage[colorlinks=true,allcolors=blue]{hyperref}
\usepackage{enumerate}
\usepackage[affil-sl]{authblk}

\newtheorem{example}{Example}
\newtheorem{proposition}{Proposition}

\graphicspath{ {./figs/} }

\newcommand{\change}[1]{{#1}}

\usepackage[round]{natbib}

\title{Model-Agnostic Meta-Learning using Runge-Kutta Methods}
\author[1]{Daniel Jiwoong Im\thanks{imd@janelia.hhmi.org}}
\author[2]{Yibo Jiang}
\author[3]{Nakul Verma}
\affil[1]{Janelia Research Campus, HHMI, Virgina}
\affil[2]{Harvard University, Massachusetts}
\affil[3]{Columbia University, New York}
\date{}
\begin{document}
\maketitle
\vspace{0.8cm}

\begin{abstract}
Meta learning has emerged as an important framework for learning new tasks from just a few examples. 
The success of any meta-learning model depends on (i) its fast adaptation to new tasks, as well as (ii) having a shared representation across similar tasks.
Here we extend the model-agnostic meta-learning (MAML) framework introduced by \citet{Finn2017} to achieve improved performance by analyzing the temporal dynamics of the optimization procedure via the Runge-Kutta method.
This method enables us to gain fine-grained control over the optimization and helps us achieve both the adaptation and representation goals across tasks. By leveraging this refined control, we demonstrate that there are multiple principled ways to update MAML and show that the classic MAML optimization is simply a special case of second order Runge-Kutta method that mainly focuses on fast-adaptation.  
Experiments on benchmark classification, regression and reinforcement learning tasks show that this refined control helps attain improved results. 


\end{abstract}
\section{Introduction}

Building an intelligent system that can learn quickly on a new task with few examples 
or few experiences is one of the central goals of machine learning. Achieving this goal requires an agent that learns continuously while having the ability to adapt to new tasks with limited data.  
Meta-learning \citep{biggs1985} has emerged as a compelling framework that strives to attain this challenging goal. 

There are two 
main approaches to meta-learning: learning-to-optimize and learning-to-initialize the 
meta-model (usually encoded as deep network). Learning-to-optimize refers to having a model that encodes the learning algorithm 
and predicts the direction of the parameter updates \citep{Hochreiter2001}. Learning-to-initialize refers to learning a representation that can quickly adapts 
to solve multiple tasks \citep{Vinyals2016, Ravi2017, Finn2017}. Here we focus on understanding and improving the 
latter approach, which is found a wide range of applications \citep{Li2017, Nichol2018, Antoniou2019}.

As one expects, over-fitting makes learning any new task from a few examples very difficult.
Meta-learning overcomes the problem of scarcity by 
jointly learning from a collection of related tasks,
each task corresponds to a learning problem on a different (but related) dataset.

Model-Agnostic Meta-Learning (MAML) has emerged as a popular state-of-the-art model in this framework \citep{Finn2017}. It is a gradient based optimization model that learns the meta-parameters (that help to generalize to new tasks) in two update phases:
 fast adaptation (inner-updates) and meta-updates 
(outer-updates). Roughly, the inner-updates optimize the parameters to maximize 
the performance on new tasks using few examples, and the outer-updates optimize the 
meta-parameters to find an effective initialization within the few parameter updates.

%


Finding a model that can yield good prediction accuracy within a few updates requires (i) fast-adaptation -- that is, finding model parameters that can either quickly change the 
internal representation (so as to maximize the sensitivity of the loss function of the new task), and/or (ii) shared-representation -- that is, developing a high quality joint latent feature representations (so as to maximize the mutual information between
different tasks).
Motivated by this, we propose new learning-dynamics for MAML optimization that gives better flexibility and improves the model on both these fronts.
Specifically, we apply the class of Runge-Kutta method to MAML optimization, which can take advantage of computing 
the gradients 
multiple-steps ahead when updating the meta-model. 
This allows us to generalize MAML to using higher-order gradients. Furthermore, 
we show that the current update rule of MAML corresponds to a specific type of second-order explicit 
Runge-Kutta method called {\em the midpoint method} \citep{Hairer1987}.

The main contribution of this work is as follows.
\begin{enumerate}[(i)]
    \item We propose a novel Runge-Kutta method for MAML optimization. This new viewpoint is advantageous as it helps get a more refined control over the optimization. (Section~\ref{sec:advect_maml})
    \item By leveraging this refined control, we demonstrate that there are multiple principled ways to update MAML and show that original MAML update rule corresponds to one of the class of second-order Runge-Kutta methods. (Section~\ref{sec:mamlrk2})
    \item The Runge-Kutta framework helps understand the MAML learning dynamics from the lens of explicit ODE integrators. To the best of our knowledge, this is the first work that successfully applies ODE solvers to 
meta-learning. (Section~\ref{sec:maml-rk})

\item We show that the refinement obtained by the Runge-Kutta method is empirically effective as well. We obtain significant improvements in performance benchmark classification, regression, and reinforcement learning tasks.
(Section~\ref{sec:experiments})
%
\end{enumerate}

\section{Model-agnostic Meta-learning (MAML)}
Model-Agnostic Meta-Learning (MAML) was introduced by \citet{Finn2017} with the aim 
to 
train a model that can adapt to a 
large set of new tasks with only a few data points in a few learning iterations. 
Given a meta-model $f$ parameterized by meta-parameters $\theta$, one wants to find a gradient learning rule can make rapid progress on a new task.
This can be formalized as follows:
for a task $\mathcal{T}_i$ drawn from a distribution of tasks $p(\mathcal{T})$, we require that
small changes in parameters using gradient based parameter updates, that is (for any given learning rate $h$)
\begin{align}
    \theta^\prime &= \theta - h \nabla_\theta \mathcal{L}_{\mathcal{T}_i}(f(\theta)),
    \label{eqn:basic_req}
\end{align}
would result in large improvement on loss function $\mathcal{L}$ of a task $\mathcal{T}_i$.
This requirement implies that meta-parameter $\theta$ is an initialization that produces
a more transferable internal representation and achieves good prediction with only a few examples from a new task. 

MAML finds such a meta-parameter by simultaneously minimizing loss functions associated with each task. 
The MAML meta-objective is defined as \citep{Finn2017}
\begin{align}
    \min_{\theta} \mathcal{L} (f(\theta^\prime)) & = 
    \min_{\theta} \sum_{\mathcal{T}_i \sim p(\mathcal{T})} \mathcal{L}_{\mathcal{T}_i}(f(\theta - h \nabla_\theta \mathcal{L} _{\mathcal{T}_i}(f(\theta)))),
    \label{eqn:maml_objective}
\end{align}
where $\theta'$ is as per Eq.\ \eqref{eqn:basic_req}, and the total loss $\mathcal{L} (f(\theta^\prime))$ is simply the aggregate of the individual task specific losses $ \sum_{\mathcal{T}_i \sim p(\mathcal{T})} \mathcal{L} _{\mathcal{T}_i}(f(\theta^\prime))$.
Thus, the meta-parameter gets updated by taking a gradient descent step in the direction that minimizes loss for all the given tasks.

Algorithm~\ref{algo:maml} provides an overview of the MAML training procedure.
A batch of tasks is sampled from the task distribution $p(\mathcal{T})$. The model parameters are then updated for each task (inner-updates). These updated parameters are then used to update the meta-parameter (outer-update).
%

For simplicity and ease of subsequent discussion, we will hide the model $f$ inside the loss $\mathcal{L}$ and refer to it as $\mathcal{L}(\theta^\prime)$ henceforth. That is, ${ \mathcal{L}(\theta^\prime) := \mathcal{L}_{\mathcal{T}_i}(\theta - h \nabla_\theta \mathcal{L} _{\mathcal{T}_i}(\theta))}$.

\begin{figure*}[t]
    \centering
    \begin{minipage}{0.493\textwidth}
        \begin{algorithm}[H]
        \caption{MAML}\label{algo:maml}
        \begin{algorithmic}[1]
        \Require{$\alpha$ and $\beta$ are learning rates}
        \Require{$p(\mathcal{T})$ is distribution over tasks}
        \State Randomly initialize meta-parameter $\theta=\theta_0$ 
        \While {not done}:
        \State Sample batch of tasks $\mathcal{T}_i \sim p(\mathcal{T})$ 
        \For {all $\mathcal{T}_i$} 
        \State Evaluate $\nabla_\theta \mathcal{L} _{\mathcal{T}_i}$ 
        \State Update model parameter: 
        \State $\theta^{\prime} = \theta - \alpha \nabla_\theta \mathcal{L}_{\mathcal{T}_i}(f(\theta))$
        \EndFor
        \State Update meta-parameter:
        \State $\theta = \theta - \beta \nabla_\theta \sum_{\mathcal{T}_i\in p(\mathcal{T})} \mathcal{L}_{\mathcal{T}_i}(f(\theta^\prime))$
        \EndWhile
        \end{algorithmic}
        \end{algorithm}
    \end{minipage}
    \begin{minipage}{0.493\textwidth}
        \begin{algorithm}[H]
        \caption{MAML-RK2}\label{algo:gmaml2}
        \begin{algorithmic}[1]
        \Require{$\alpha$ and $\beta$ are learning rates}
        \Require{$p(\mathcal{T})$ is distribution over tasks}
        \State Randomly initialize meta-parameter $\theta=\theta_0$ 
        \While {not done}:
        \State Sample batch of tasks $\mathcal{T}_i \sim p(\mathcal{T})$ 
        \For {all $\mathcal{T}_i$} 
        \State Evaluate $\nabla_\theta \mathcal{L} _{\mathcal{T}_i}$ 
        \State Update model parameter: 
        \State $\theta^{\prime} = \theta -  \textcolor{red}{ h q_{21}} \nabla_\theta \mathcal{L}_{\mathcal{T}_i}(f(\theta))$
        \EndFor
        \State Update meta-parameter:
        \State  $\theta = \theta - \textcolor{red}{h} \sum_{\mathcal{T}_i\in p(\mathcal{T})} \textcolor{red}{ (a_1 \nabla_{\theta} \mathcal{L}_{\mathcal{T}_i}(f(\theta)))} + 
        \textcolor{red}{ a_2\nabla_\theta \mathcal{L}_{\mathcal{T}_i}(f(\theta^\prime)))}$
        \EndWhile
        \end{algorithmic}
        \end{algorithm}
    \end{minipage}        
\end{figure*}


\section{Generalizing MAML using the Runge-Kutta method}
\label{sec:maml-rk}

One of the central objectives of this paper is understanding the learning dynamics of MAML
through the lens of explicit ODE integrators. 
Consider the vector field $\mathcal{X}$ that maps the space of parameters to the descent 
directions induced by the gradient of the MAML objective. So long as $\mathcal{X}$ is sufficiently smooth, we can look for solutions of the form 
\[
    \frac{d\theta}{dt} = \mathcal{X}(\theta).
\]
The temporal dynamics of this ODE constitutes the ideal path that the meta-learner takes during training given 
an initial parameter $\theta_0$. We can therefore use a numerical ODE solver to solve approximate this ideal solution. 
Clearly a higher order explicit ODE integrator can be viewed as a black 
box that transforms $\mathcal{X}$ into a new vector field $\mathcal{\bar X}$ such that
the parameters $\theta$ evolve more closely to the streamlines of $\mathcal{X}$.
For gradient descent, $\mathcal{X} = \nabla_{\theta} \mathcal{L} $ and the resulting 
$\mathcal{\bar X}$ advects $\theta$ more closely along the path of the steepest descent.
One of the central goals of this work is to precisely to discern the efficacy of applying integrators on
$\mathcal{\bar X}$ rather than $\mathcal{X}$. To do this we will replace the 
gradient $\mathcal{X}(\theta_t) = \nabla_\theta \mathcal{L}(\theta_t)$ with the appropriate 
$\mathcal{\bar X}(\theta_t)$ by calling a chosen explicit ODE integrator. 
\change{This viewpoint generalizes the MAML optimization framework to optimize with respect to temporal parameters in contrast to considering the spatial parameters as done in the previous literature} \citep{Park2019, Chen2018, Song2019}. 
%
%
%

For a given timestep $t$, explicit integrator can be seen as a morphism over vector fields 
$\mathcal{X} \rightarrow \mathcal{\bar{X}}^h$ \change{(for a fixed stepsize $h$)}. 
Hence, for a true gradient $g_t= \nabla_\theta \mathcal{L}(\change{\theta_t})$ \change{(at time $t$)}, we solve the modified Runge-Kutta 
gradient $\bar{g_t} = \nabla_\theta \mathcal{L}(\change{\theta^\prime_t})$ as follows. 
Define
\begin{align}
    \texttt{advect}_{g_t}^{\textup{RK}}(\theta,h):= \theta_t + \bar{g_t}h = \theta_{t+1} 
    \nonumber
\end{align}
The general form of $\texttt{advect}_{g_t}^{\textup{RK}}(\theta,h)$ is the Runge-Kutta (RK) equation \change{of order $N$} \citep{Butcher2008}, given by 
\allowdisplaybreaks
\begin{align}
    \theta_{t+\change{1}} &= \theta_t + h\sum^{N}_{i=1} a_i k_i, \hspace{0.35in} \textrm{where}
\end{align}
\begin{align}
    k_1 :=& \nabla_{\theta} \mathcal{L}(t, \theta_t), \nonumber\\
    k_2 :=& \nabla_{\theta} \mathcal{L}(t + p_2h, \theta_t + q_{21} k_1 h),\nonumber\\
    k_3 :=& \nabla_{\theta} \mathcal{L}(t + p_3h, \theta_t + q_{31} k_1 h + q_{32} k_2 h),\nonumber\\
     & \vdots \nonumber\\
    k_N :=& \nabla_{\theta} \mathcal{L}(t + p_nh, \theta_t + q_{n1} k_1 h + q_{n2} k_2 h 
        + \cdots + q_{n,n-1} k_n h),
    \label{eqn:general_rk}
\end{align}
where (i) $a_i$ are combination weights (that should sum to $1$), 
 (ii) $p_j$ are the so called \emph{nodes} that scale the timestep (for better numerical approximation), (iii) $q_{ij}$ are the coefficients that scale the step towards the gradient $k_i$ (for a fine grain optimization control), and (iv) $\sum^{i-1}_{j=1} q_{ij}=p_j$ for all $j=2,\ldots,N$.
\change{The specific choice of these parameters ($a_i$, $p_j$ and $q_{ij}$) gives rise to various popular instantiations of the RK optimization method}.

Figure~\ref{fig:rk_slopes} demonstrates the slopes $k_i$ for a quadratic function (depicted in blue color) for $1\leq i\leq N=4$.
Notice that it takes a linear combination of the $k_i$'s (red arrows) to calculate the next step vector (green arrow). These $k_i$'s can be thought as \emph{forward multi-steps} since they compute the gradients at future timesteps (c.f.\ Figure \ref{fig:rk_slopes}).
Rearranging the terms with respect to $\bar{g_t}$, we get
\begin{align}
    \bar{g_t}^{\text{MAML-RK}}  :=& \frac{\texttt{advect}_{g_t}^{\textup{RK}}(\theta,h) - \theta_t}{h} 
                = \frac{\theta_t+ h\sum^{N}_{i=1} a_i k_i - \theta_t}{h}
                =
                \sum^{N}_{i=1} a_i k_i. \nonumber
\end{align}
We can use this refined Runge-Kutta gradient in the MAML optimization for better convergence and will call it as the {\textbf{generalized Runge-Kutta MAML}} (MAML-RK) method.
%
Observe that $\bar{g}_t^{\text{MAML-RK}}$ is a linear combination of the gradients of the forward multi-steps $k_i$
over the sum of individual task specific loss functions $\mathcal{L}_{\mathcal{T}_i}(t,\theta_t)$. 
Unlike the standard MAML update, which takes the gradient only one step forward (inner-update)\footnote{The meta 
optimization is performed over the meta-parameter $\theta$ while the objective is computed using the updated 
model parameters $\theta^\prime$.}, MAML-RK has the ability to take the gradient multiple steps ahead. This 
encourages the meta-learner to find parameters that adapt to new tasks with very few (sometimes even just one) gradient 
steps on a new task, just as what we desire. Note that although we advect $g_t$ with RK in this work, any other other explicit 
ODE solver can also be used. 

Figure~\ref{fig:gmaml_ill} illustrates of the optimization path of MAML-RK. 
The bold blue path represents the continuous learning dynamics of the meta-parameter $\mathcal{X}(\theta)$.
The three parameters $\theta_1$, $\theta_2$, and $\theta_3$ are the optimal parameters
for three tasks, and the best meta-parameter $\theta^*$ is chosen in such a way that it lies close to all of them.
The path of standard MAML optimization (the green curve) deviates from $\mathcal{X}(\theta)$ due local crude approximations. 
Using higher-order explicit RK methods (the red curve), follows the ideal path better due to better quality approximation.

Next we show that the original MAML optimization is a special case of MAML-RK \change{(with a specific setting of $a_i$, $p_j$, $q_{ij}$)} and also explore a wider setting of these parameters giving us other types of
first and second-order optimizations for MAML.

\begin{figure*}[t]
    \centering
    \begin{minipage}{0.47\textwidth}
        \includegraphics[width=\linewidth]{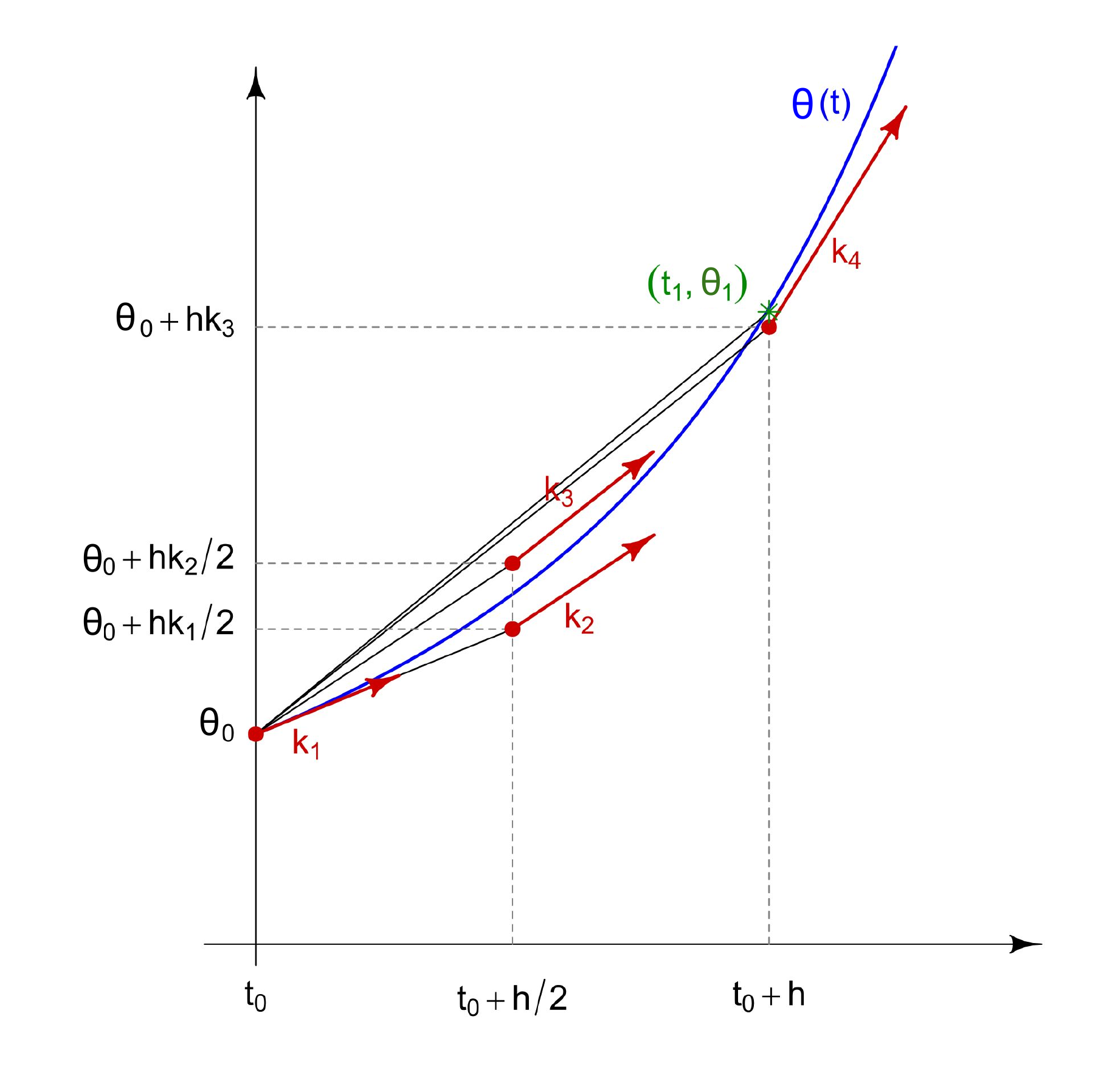} 
        \vspace{-0.8cm}
        \caption{Runge-Kutta method slope $k_i$ illustrations.\protect\footnotemark
        Update direction from $\theta_0$ to $\theta_1$ is composed of linearly combining different $k_i$s.
        }
        \label{fig:rk_slopes}
    \end{minipage}
    \begin{minipage}{0.52\textwidth}
        \includegraphics[width=\linewidth]{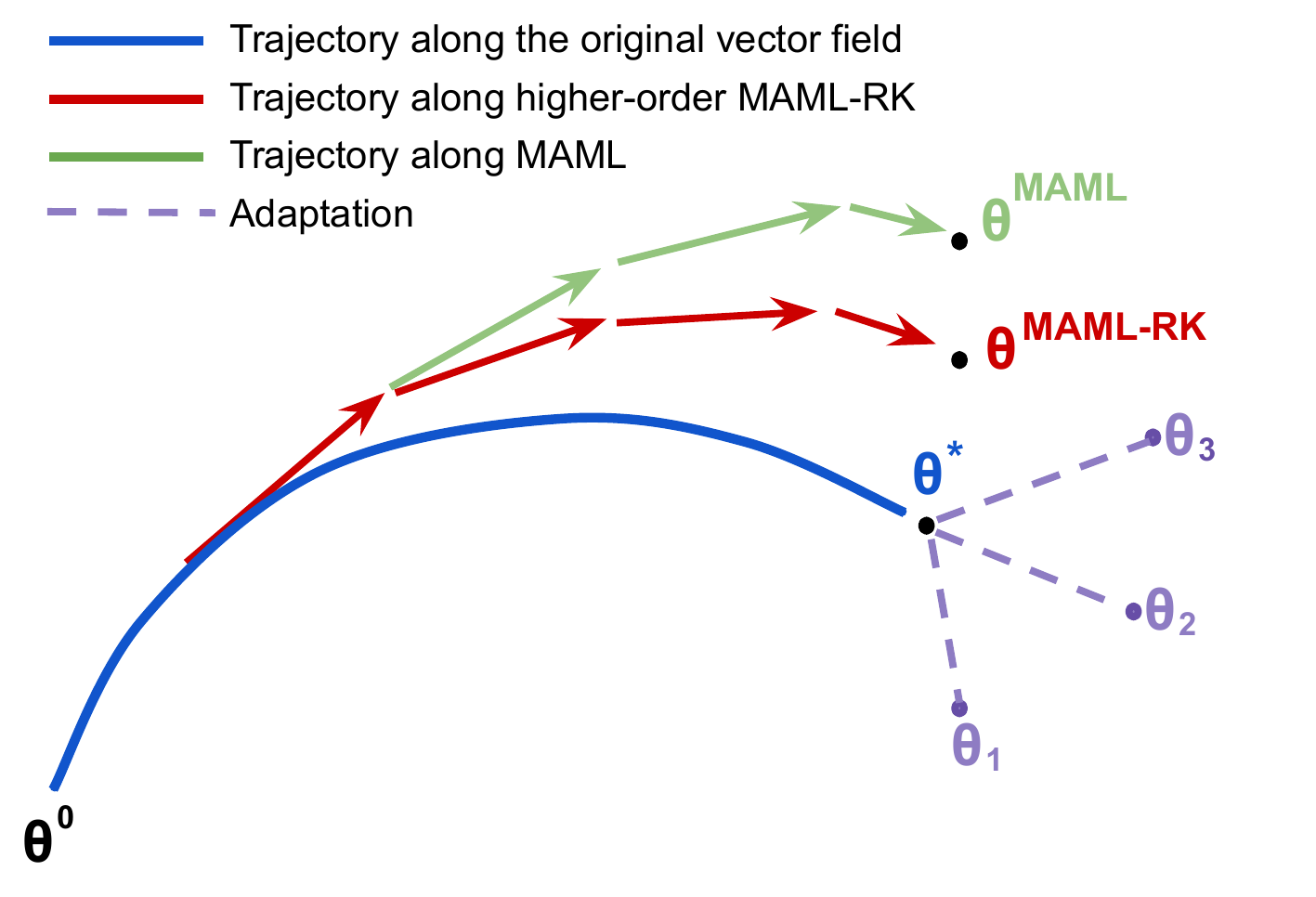}        
        \caption{Diagram of MAML-RK illustrates how parameters $\theta$ evolve over time.
        The higher-order MAML-RK takes the path that is closer to continuous path $\mathcal{X}$ (bold red curve)
        that is closer to best parameter initialization that can quickly adapts to multiple tasks.}
        \label{fig:gmaml_ill}
    \end{minipage}
\end{figure*}
\footnotetext{The original image is from Runge-Kutta Wikipedia: \texttt{en.wikipedia.org/wiki/File:Runge\-Kutta\_slopes.svg}}

\subsection{Advect MAML's gradient}
\label{sec:advect_maml}
The technical key component of MAML is that it computes the gradient with respect to the meta-parameter $\theta$ 
while computing the objective on the updated model parameter $\theta^\prime$ (c.f.\ Equation~\ref{eqn:maml_objective} and Algorithm \ref{algo:maml}). 
This makes the meta-parameter to move towards the direction of $\nabla_\theta \sum_{\mathcal{T}_i\sim p(\mathcal{T})} \mathcal{L}_{\mathcal{T}_i} (\theta_i^\prime)$.
Observe that this precisely corresponds to the $k_2$ component of MAML-RK with a specific setting of $p_2$ and $q_{21}$.
The following proposition tells us that MAML is a special case of {\em second-order} MAML-RK.
\begin{proposition}
    The MAML's gradient corresponds to the second-order explicit Runge-Kutta equation with
    the parameters $a_1=0$, $a_2=1$, $q_{21}=\frac{1}{2}$, $p_2=\frac{1}{2}$. 
    \label{prop:maml_midpt}
\end{proposition}
\begin{proof}
For simplicity, let us denote $\nabla_\theta \mathcal{L}(t,\theta_t)$ as $\nabla \mathcal{L}(t,\theta_t)$.
Consider the Taylor expansion of MAML's gradient:
\allowdisplaybreaks
\begin{align}
    \frac{d\theta}{dt}
    &= \nabla \mathcal{L} \left(t+h, \theta_t - h \nabla \mathcal{L} (t,\theta_t)\right)\nonumber\\
    &= \nabla \mathcal{L} \left(t, \theta_t\right) + h \Bigg(\frac{d\nabla\mathcal{L}(t, \theta_t)}{dt}
     -\frac{\partial\nabla \mathcal{L}(t, \theta_t)}{\partial \theta}\nabla \mathcal{L}(t, \theta_t)  \Bigg) + \mathcal{O}(h^2)
    \label{eqn:maml_taylor}.
\end{align}

We can compare Equation~\ref{eqn:maml_taylor} with the second-order explicit Runge-Kutta Equation for MAML, which is 
\begin{align}
    \theta_{t+h} &= \theta_t + (a_1\nabla \mathcal{L}(t, \theta_t)
       + a_2\nabla \mathcal{L}(t + p_2h, x_t + q_{21} k_1 h)) h\nonumber \\
    &= \theta_t + (a_1+a_2) \nabla\mathcal{L}(t, \theta_t)  + a_2 h \Bigg(\frac{d\nabla \mathcal{L}(t, \theta_t)}{dt}p_2  + \frac{\partial\nabla \mathcal{L}(t, \theta_t)}{\partial \theta}\nabla \mathcal{L}(t, \theta_t) q_{21}\Bigg) + \mathcal{O}(h^2), \label{eqn:second_order_rk}
\end{align}
such that $a_1 + a_2 = 1, a_2 q_{21} = \frac{1}{2}, a_2 p_1 = \frac{1}{2}$.

We can now see that Equation~\ref{eqn:second_order_rk} equals Equation~\ref{eqn:maml_taylor} with $a_1=0$, $a_2=1$, $p_2=\frac{1}{2}$, and $q_{21}=\frac{1}{2}$.
%
\end{proof}

The optimization done with the specific setting of the RK parameters as shown in Proposition \ref{prop:maml_midpt} is  usually called the midpoint method \citep{Butcher2008}. This shows that the original MAML objective is essentially a midpoint optimization method.

A setting of $a_1=0$ and $a_2=1$ implies that the classic MAML objective solely relies on $\nabla \mathcal{L}_{\theta}(t + p_2h, x_t + q_{21} k_1 h)$.
Moreover, a setting of $q_{21} = \frac{1}{2}$ corresponds to a learning rate of inner-update (i.e.\ $\alpha$ in Algorithm \ref{algo:maml}) to be $\frac{1}{2}h$ and 
the learning rate of outer meta-update (i.e.\ $\beta$ in Algorithm \ref{algo:maml}) to be $h$. This results in a  meta-model optimization to stays as close to the ideal trajectory  
path $\mathcal{X}(\theta)$ with an error rate of $\mathcal{O}(h^3)$ per step (c.f.\ proof of Proposition \ref{prop:maml_midpt}).

\subsection{Examples beyond MAML}
\label{sec:mamlrk2}

The Runge-Kutta gradient $\bar{g}_t^{\textup{MAML-RK}}$ discussed so far has been generic. 
One can instantiate it (i) at various degrees of order (by choosing $N$), and (ii) by varying the RK parameters in 
Equation~\ref{eqn:general_rk}. In this section, we examine some example instantiations of MAML-RK.

\begin{example}[First-order MAML]
Choose $N=1$ and $a_1=1$. Then,
\begin{equation}
    \bar{g}_t^{\textup{MAML-RK1}} = h \nabla_{\theta} \sum_{\mathcal{T}_i \sim p(\mathcal{T})} \mathcal{L}_{\mathcal{T}_i}(t, \theta_t).
\end{equation}
\end{example}

\begin{wraptable}{r}{0.5\textwidth}
   \centering
   \captionof{table}{The coefficients of various $2^{\textup{nd}}$-order RK methods \citep{Hairer1987}}
   \label{tab:rk2_methods}
   \begin{tabular}{| c | c | c | c |}
       \hline
       Methods & $a_1$ & $a_2$ & $q_{21},p_2$ \\[1ex]
       \hline\hline
       \rule[1.1ex]{0pt}{1ex}
       Midpoint & $0$ & $1$ & $\frac{1}{2}$ \\[.1ex]
       \hline
       \rule[1.1ex]{0pt}{1ex}
       Heun& $\frac{1}{2}$ & $\frac{1}{2}$ & $1$ \\[.1ex]
       \hline
       \rule[1.1ex]{0pt}{1ex}
       Ralston& $\frac{1}{3}$ & $\frac{2}{3}$ & $\frac{3}{4}$\\[.1ex]
       \hline
       \rule[1.1ex]{0pt}{1ex}
       ITB & $\frac{2}{3}$ & $\frac{1}{3}$ & $\frac{3}{2}$ \\[.1ex]
       \hline
       Generic & $\frac{1}{2x}$ & $1-\frac{1}{2x}$ & $x$ \\[.1ex]
       \hline
   \end{tabular}
\end{wraptable}
The first-order MAML is simply the Euler's method. It sums over 
the gradients of the  loss functions for every task. It is similar in flavor to other first-order approaches, such as FOMAML \citep{Biswas2018} and 
Reptile \citep{Nichol2018}. 

In the previous section, we showed that MAML is a special case of second-order MAML-RK by
advecting $g_t$ with MAML's gradient. Here, we illustrate several other popular second-order
Runge-Kutta methods that we can also apply to MAML gradients, thus extending the variety of optimization techniques that are currently in use for meta-learning.

\begin{example}[Second-order MAMLs]
Choose $N=2$. The methods in Table~\ref{tab:rk2_methods} satisfy the constraints $a_1+a_2=1$, $a_2q_{21}=\frac{1}{2}$, 
and $a_2p_2=\frac{1}{2}$. Thus
\begin{align}
    \bar{g}_t^{\textup{MAML-RK2}}& = a_1\nabla_{\theta} \mathcal{L}(t, \theta_t)
        + a_2\nabla_{\theta} \mathcal{L}\left(t + p_2h, (\theta_t + \nabla_{\theta} \mathcal{L}(t, \theta_t)) q_{21} h\right),
\end{align}
where the parameters $a_1$, $a_2$, $p_2$, and $q_{21}$ can be substituted accordingly.
\end{example}
We can thus generalize Algorithm \ref{algo:maml} by substituting the learning rate $\alpha$ as $q_{21}h$, and $\beta$ as $h$.
Note that $q_{21} = p_2$ for all second-order Runge-Kutta methods. 
Algorithm~\ref{algo:gmaml2} presents generalized MAML-RK2 training algorithm.
It is worth noting that the existing literature only discusses  the midpoint method for training MAML, and our extension enables the practitioner to explore 
other methods, like Heun's, Ralston, and ITB \cite{Hairer1987}.

\subsection{Additional Remarks}
Using the gradient ($\bar{g}^{\textup{MAML-RK1}}=\nabla \mathcal{L}(\theta)$) from a pre-trained model that uses a large dataset and then fine-tuning it on a smaller new dataset is popular in transfer learning and has become a popular techniques in various application domains. For example in computer vision, it is common to use a parameters and gradients of a pre-trained  network on ImageNet \citep{Deng2009} and use it to perform, say, bird classification \citep{Zhang2014}.
Hence, there is some evidence that
pre-training with first order gradient, i.e.\ $\bar{g}^{\textup{MAML-RK1}}$ helps the model learn a shared feature 
representation that can be applied across similar tasks. Although it is worth noting that it does not
encourage the model to learn a meta-parameter that can rapidly adapt to a new task. 
In contrast, part of the success of the classic MAML optimization 
 $\bar{g}^{\textup{MAML}}$ (a specific second order method) is because it can directly optimize for this rapid new task adaptation by differentiating 
through the fine-tuning process with respect to the meta-parameter 
$\nabla \mathcal{L}(\theta^\prime)$, but it does not make use of the first order gradient 
$\nabla \mathcal{L}(\theta)$. 
it is important to note our second order Runge-Kutta generalization $\bar{g}^{\textup{MAML-RK2}}$ (with specific instantiations as Heun's, Ralston, and ITB), considers
both the terms: $\nabla \mathcal{L}(\theta)$ and $\nabla \mathcal{L}(\theta^\prime)$, and thus have the potential benefit
of encouraging both rapid adaptation and shared feature representation \citep{Raghu2019}. See our experiments in  Section~\ref{sec:experiments}.

For every meta-learning update, the first-order method performs one evaluation of $\mathcal{L}$,
 and the second-order method performs two evaluations for RK methods. The number of evaluations grows linearly 
up to the fourth-order, after which it grows faster making it computationally prohibitive.
The fourth-order Runge-Kutta method is often the 
popular method for solving initial value ODE problems. In our case, even the second-order 
optimization requires Hessian-vector products during the MAML updates and more evaluations is impractical. 
We will therefore limit to second order RK methods in  
our experiments. 

\section{Related Work}

\begin{table}[b]
  \centering
  \caption{The performance of MAML-RK for sinunoid regression tasks on 
    10-shot adaptation problem. MAML-RK1 (the first-order method) corresponds to
    standard pre-training model on all training tasks. MAML-RK2 (midpoint)
    corresponds to stanard MAML method. }
  \label{tab:regression_results}
  \begin{tabular}{|c|c|c|}\hline
      \multicolumn{2}{|c|}{Sinuoid} & 10-shot \\\hline\hline
            & MAML-RK1 (pretrained) & 2.72 $\pm$ 0.174\\
            & MAML-RK2 (midpoint)     & 0.13 $\pm$ 0.038\\\hline
      \parbox[t]{2mm}{\multirow{3}{*}{\rotatebox[origin=c]{90}{ours}}} 
            & MAML-RK2 (Heun's)     & 0.19 $\pm$ 0.053\\
            & MAML-RK2 (Ralston)    & {\bf 0.12 $\pm$ 0.039}\\
            & MAML-RK2 (ITB)        & 0.18 $\pm$ 0.054\\\hline
  \end{tabular}
    \vspace{-0.25cm}
\end{table}

Early approaches to meta-learning goes back to the late 1980s and early 1990s \citep{Schmidhuber1987, Bengio1991}
where it studies evolutionary principles in self-referential learning using
genetic programming. There has been a recent surge in interest where meta-models are applied to a wide array of tasks from architecture search \citep{Zoph2016} and hyperparameter search \citep{Maclaurin2015},
to learning optimization \citep{Chen2017}. It has become a popular approach in 
few-shot supervised learning \citep{Hariharan2016} and fast reinforcement learning \citep{Wang2017}.

Among many approaches to meta-learning, \citet{Finn2017} proposed the MAML framework that uses a meta-objective
function that performs a two-step gradient-based optimization of the meta-parameters for fast task-specific optimization. 
Various studies \citep{Li2017, Antoniou2019} show that a fast-adaptation 
update rule can heavily influence performance, which has initiated several related investigations.
\citet{Biswas2018} and \citet{Nichol2018} use first-order update methods to reduce computational burden, while
\citet{Park2019}, \citet{Chen2018} and \citet{Song2019}
leverage second-order curvature information for better generalization.
Instead of analyzing the spatial features (such as gradient and curvature of the parameter space), our work focuses on the temporal dynamics of the optimization. 

Recent work by \citet{Raghu2019} analyzes the source of effectiveness of MAML -- is it primary due to {\em rapid learning}, or due to {\em feature reuse}? Their analysis shows MAML partly achieves both -- large portion of the lower layers of the MAML model helps in feature reuse and 
large portion of the upper layers helps in rapid learning. Part of the goal of our work is to think about aspects of a meta-model more effectively. Our RK extension makes 
fast-adaptation and shared-representation more explicit giving practitioners a fine grain control over the optimization.


\begin{figure*}[t]
    \centering
    \begin{minipage}{0.24\textwidth}
        \includegraphics[width=\linewidth]{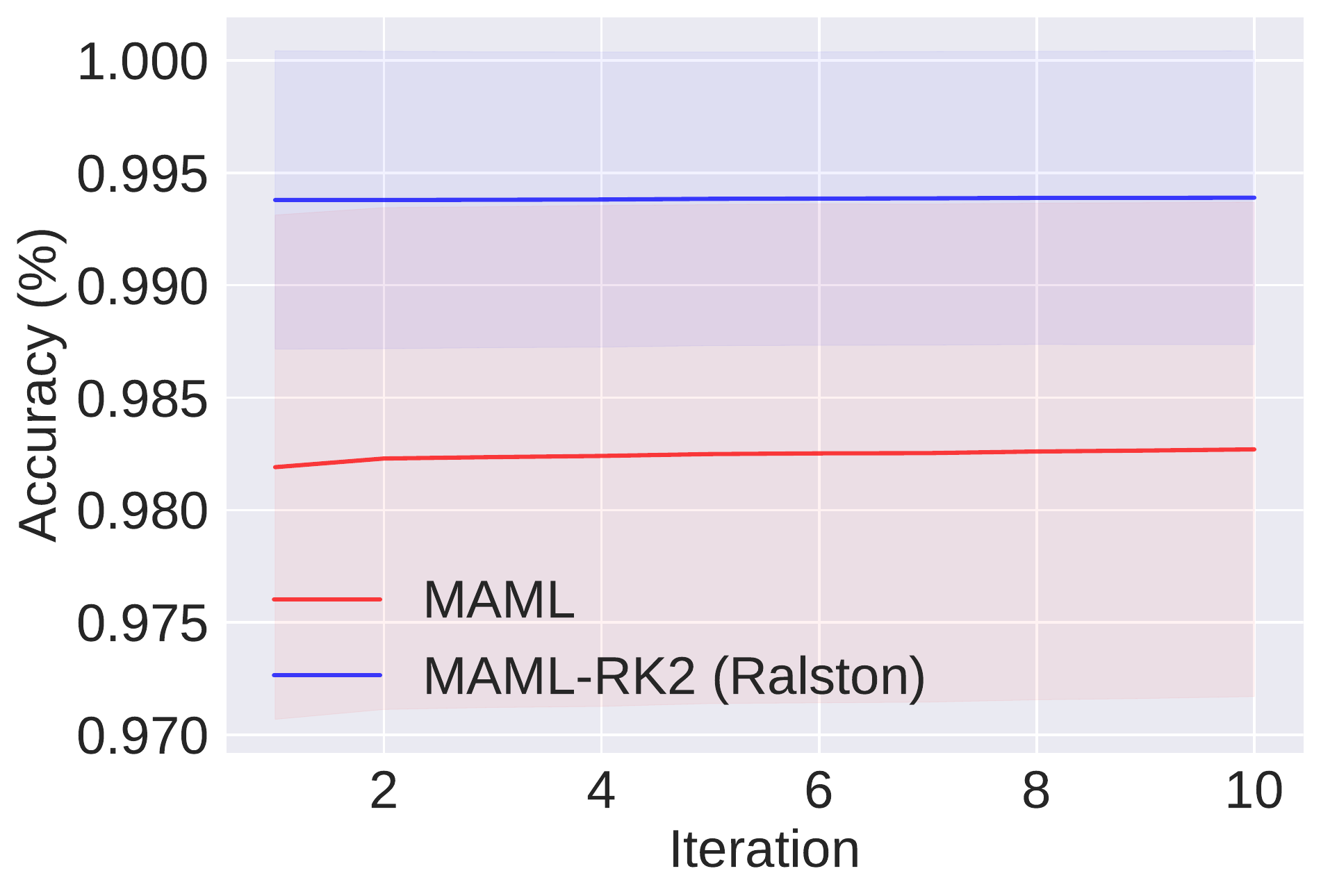} 
        \subcaption{5way-1shot}
    \end{minipage}
    \begin{minipage}{0.24\textwidth}
        \includegraphics[width=\linewidth]{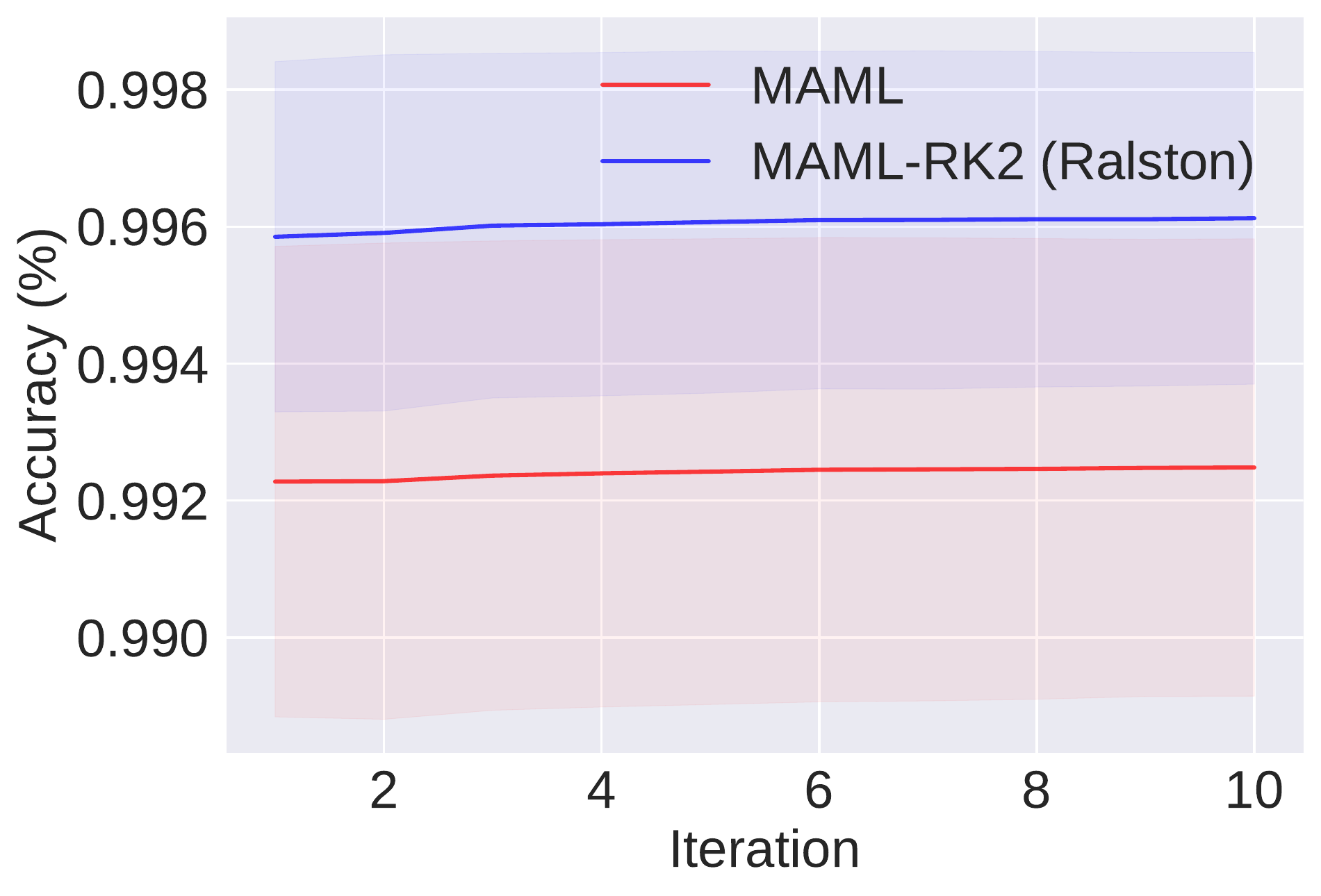} 
        \subcaption{5way-5shot}    
    \end{minipage}
    \begin{minipage}{0.24\textwidth}
        \includegraphics[width=\linewidth]{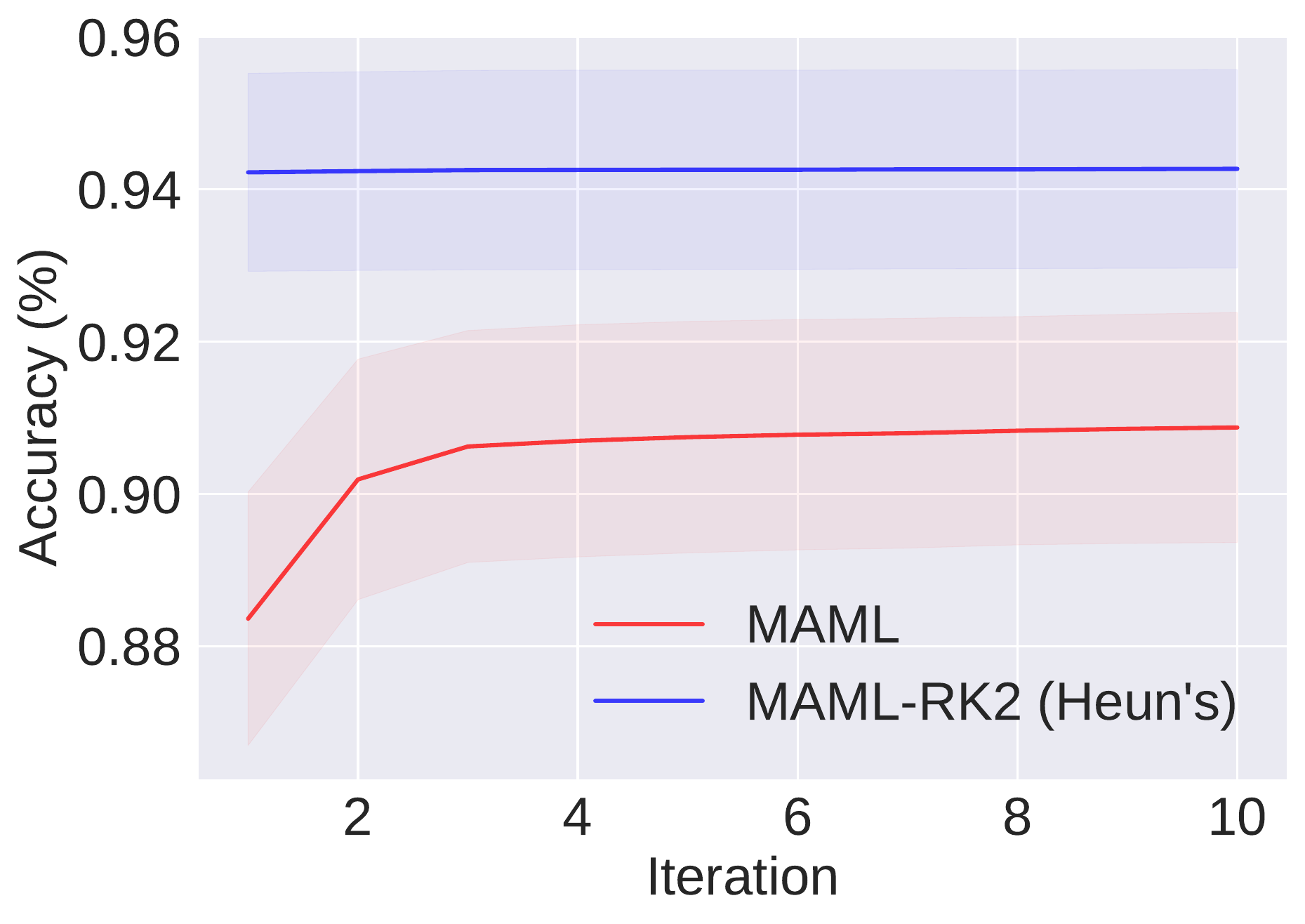} 
        \subcaption{20way-1shot}
    \end{minipage}
    \begin{minipage}{0.24\textwidth}
        \includegraphics[width=\linewidth]{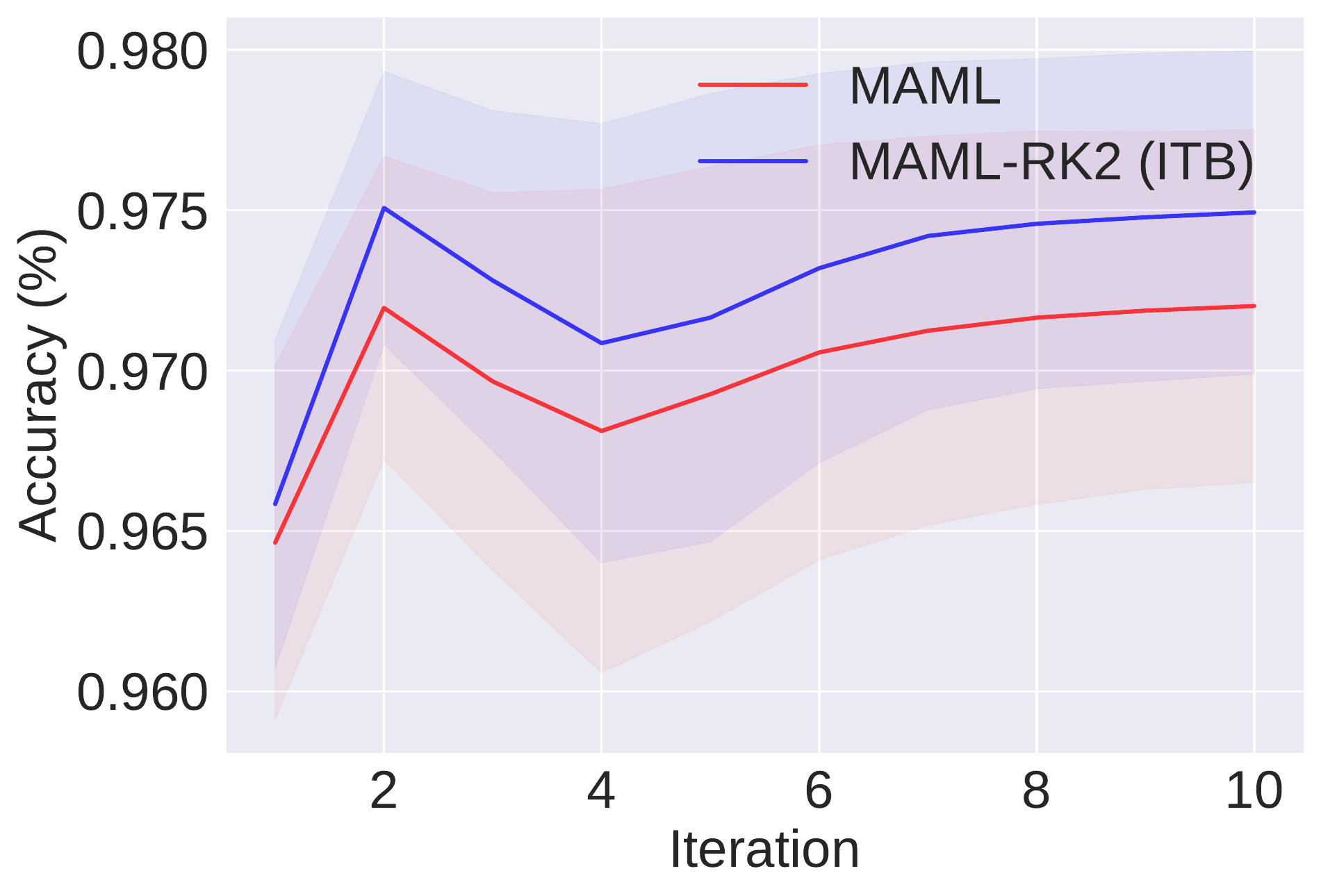} 
        \subcaption{20way-5shot}    
    \end{minipage}
    \caption{Fast-adaptation of MAML-RK on Ominiglot test datsets.}
    \label{fig:perm_iter}
\end{figure*}
\begin{table*}[t]
  \vspace{-0.25cm}
  \centering
  \caption{The performance of MAML-RK for Omniglot and MiniImagenet image classification tasks on 
    1- and 5-shot adaptation problems. The midpoint method corresponds to MAML.}
  \label{tab:omniglot_results}
  \begin{tabular}{|c|c|c|c|c|c|}\hline
      \multicolumn{2}{|c|}{\multirow{2}{*}{Omniglot}}   & \multicolumn{2}{c|}{5-way}  & \multicolumn{2}{c|}{20-way}  \\\cline{3-6}
      \multicolumn{2}{|c|}{}                    & 1-shot & 5-shot       & 1-shot                 & 5-shot \\\hline
      \hline & Midpoint & 98.26 $\pm$ 1.09      & 99.24 $\pm$ 0.33      & 90.87 $\pm$ 1.51       & 97.20 $\pm$ 0.55 \\\hline
      \parbox[t]{2mm}{\multirow{3}{*}{\rotatebox[origin=c]{90}{ours}}} 
                       & Heun's   & 99.24 $\pm$ 0.73      & 99.09 $\pm$ 0.66 & 94.27 $\pm$ 1.30       & {\bf 97.48 $\pm$ 0.53} \\
                       & Ralston  & \bf{99.39 $\pm$ 0.65} & {\bf 99.61 $\pm$ 0.24}       & {\bf 94.91 $\pm$ 1.19} & 97.45 $\pm$ 0.54\\
                       & ITB      & 98.93 $\pm$ 0.89      & 99.12 $\pm$ 0.67       & 93.81 $\pm$ 1.33       & 97.49 $\pm$ 0.50 \\\hline
  \end{tabular}
  \centering
  \begin{tabular}{|c|c|c|c|}\hline
      \multicolumn{2}{|c|}{\multirow{2}{*}{MiniImagenet}} & \multicolumn{2}{c|}{5-way} \\\cline{3-4}
      \multicolumn{2}{|c|}{}    & 1-shot & 5-shot  \\\hline\hline
                                & Midpoint  & 45.21 $\pm$ 5.20 & 59.92 $\pm$ 5.17 \\\hline
      \parbox[t]{2mm}{\multirow{3}{*}{\rotatebox[origin=c]{90}{ours}}} 
                                & Heun's    & {\bf 46.65 $\pm$ 5.10} & {\bf 60.40 $\pm$ 5.15} \\
                                & Ralston   & 44.66 $\pm$ 5.10 & 59.95 $\pm$ 5.16 \\
                                & ITB       & 45.29 $\pm$ 5.38 & 60.19 $\pm$ 4.85 \\\hline
  \end{tabular}
  \vspace{-0.25cm}
\end{table*}

\section{Experiments}
\label{sec:experiments}

To study the effectiveness of our extended Runge-Kutta MAML framework, we conducted detailed empirical studies of various explicit instantiations of second order MAML-RK methods (as detailed in 
in Section~\ref{sec:maml-rk}) on various classification-, regression- and reinforcement- meta-learning benchmarks. 
Throughout the experiments, we compare the midpoint (i.e.\ the original MAML optimization), Heun's, Ralston, ITB methods.
The data, models, and the optimizer for all our experiments is built upon the original MAML code\footnote{MAML regression and classification code: \texttt{http://github.com/cbfinn/maml}
and MAML reinforcement learning code: \texttt{http://github.com/cbfinn/maml\_rl}.}.

The following standard setup is used in all our experiments. 
All models are trained from the training task dataset and
evaluated on the test task dataset. For each task, we have a support set of $K$ examples, which are used
for fast-adaption updates.  During the evaluation phase, the model is initialized with the learned meta-
parameters from training phase, and is fine-tuned on the $K$ samples from the test tasks. The model architecture
for each experiment can be found in the Appendix.

\textbf{Regression} - Following the experiments in MAML \citep{Finn2017}, we consider the sinusoid regression problem. For each task, the 
1-dimensional sinusoid wave of amplitude and phase are varied between $[0.1, 5.0]$ and $[0,\pi]$, and
the goal is to regress on an unseen sinusoid wave. The datapoints are sampled from the range of $[-5.0,5.0]$ and 
we used batch size of ten ($K=10$) for every gradient update with a fixed step size of $0.01$. 
The mean-square-error is used as a loss function. 
Table~\ref{tab:regression_results} presents the regression performance.
The MAML-RK1 (the first-order method) corresponds to the pre-trained model on all training tasks, which is the
baseline method. Note that MAML-RK2 (midpoint) corresponds to original MAML method.
We observe that MAML-RK2 (Ralston) performs marginally better than the midpoint method for this simple task. 

\textbf{Classification} - Next, we evaluate MAML-RK on classification tasks. It has been shown that MAML \citep{Finn2017} achieved state-of-the-art performance when compared to prior meta-learning and few-shot learning algorithms \citep{Koch2016, 
Vinyals2016, Ravi2017} on few-shot Omniglot \citep{Lake2011} and MiniImagenet \citep{Ravi2017} image 
 recognition tasks. Therefore, we only report MAML's results against other extended MAML-RK2 methods.
The standard setup for few-shot classification is that we consider $5$ and $20$ 
classes ($N$-way) with 1-shot and 5-shot learning, and evaluate on new $5$ and $20$
classes. 

 The Omniglot images were downsampled to $28\times 28$ and were  augmented with up to 90
 degrees of rotation. The training task classes were randomly selected from 1200 out of 1623 
characters and rest were used as test task classes.
We followed the same model architectures from previous studies \citep{Vinyals2016, Finn2017}.
For MiniImage classification, the dataset consists of
84$\times$84 60,000 colored images. There are 100 different classes where each class
consists of 600 images. Out of the 100 classes, the training, validation, and test 
classes were split as 64, 12, and 24 respectively. 

Table~\ref{tab:omniglot_results} presents the performance on Omniglot dataset for
5-way and 20-way classification for 1-shot and 5-shot learning.
For 5-way 1-shot and 20-way 1-shot learning, we observe that Ralston method performs 
the best, and followed by Heun's and ITB. The midpoint method performs the worst. 
We suspect that this is because only the midpoint method has zero coefficient for $a_1$, 
which means that it cannot take the first-order gradient 
information $\nabla \mathcal{L}(\theta)$ into account. Hence, setting the coefficient $a_1$ to be greater 
than $0$ is important for achieving better results on this dataset.
Again for 5-way 5-shot learning, we observe that Ralston method performs the best and
the midpoint method performs the second best. 
Figure~\ref{fig:perm_iter} (a-c) illustrates that Ralston outperforms MAML
throughout the training process. 
For 20-way 5-shot learning, we found that all MAML-RK2 models perform
more or less the same, where the range of mean plus and minus the standard deviation
overlaps. 
Heun's method performed the best for the MiniImagenet dataset.

\begin{figure*}[t]
    \centering
    \begin{minipage}{0.244\linewidth}
        \includegraphics[width=\linewidth]{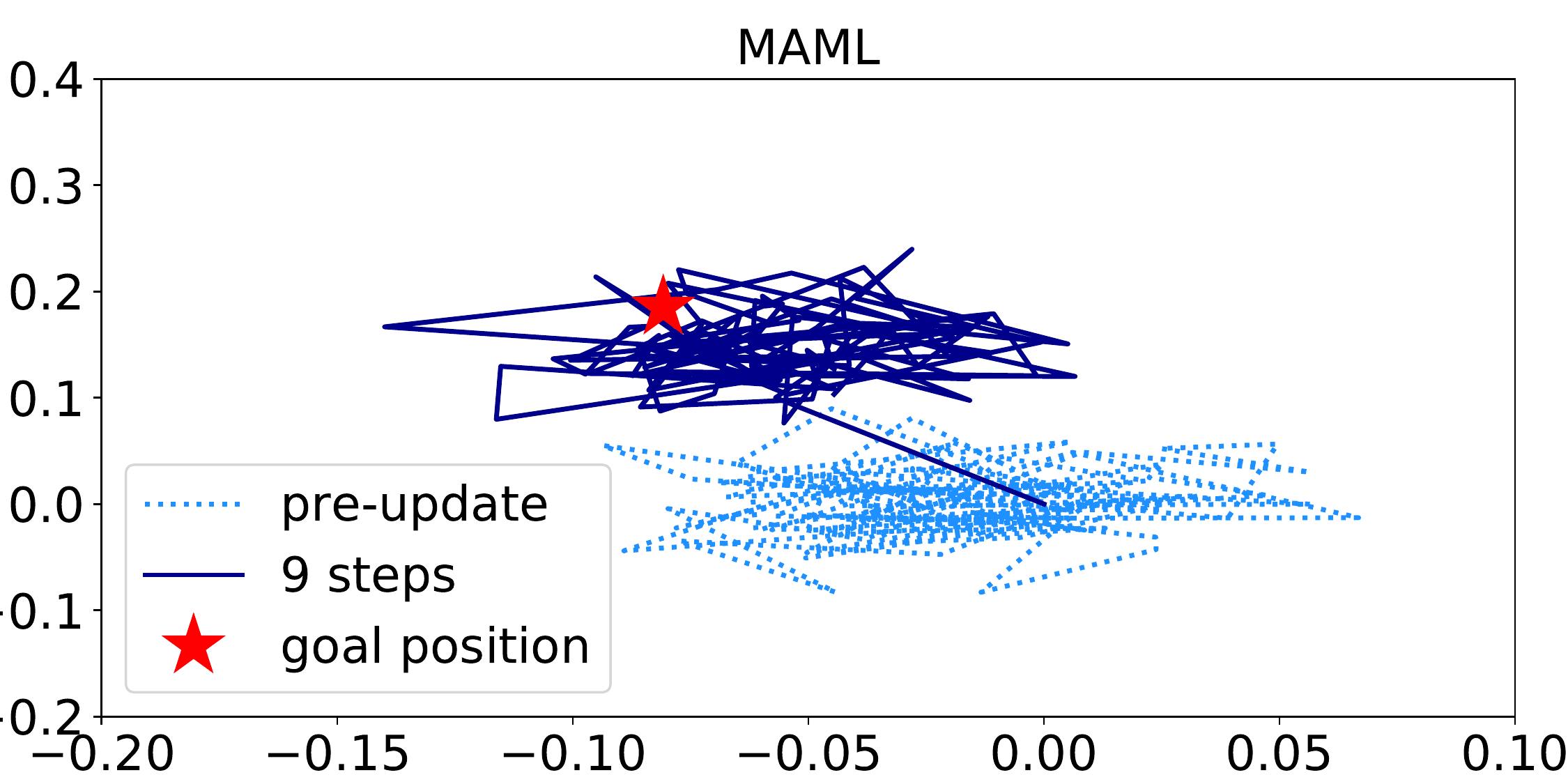} 
        \subcaption{Midpoint (MAML)}
    \end{minipage}
    \begin{minipage}{0.244\linewidth}
        \includegraphics[width=\linewidth]{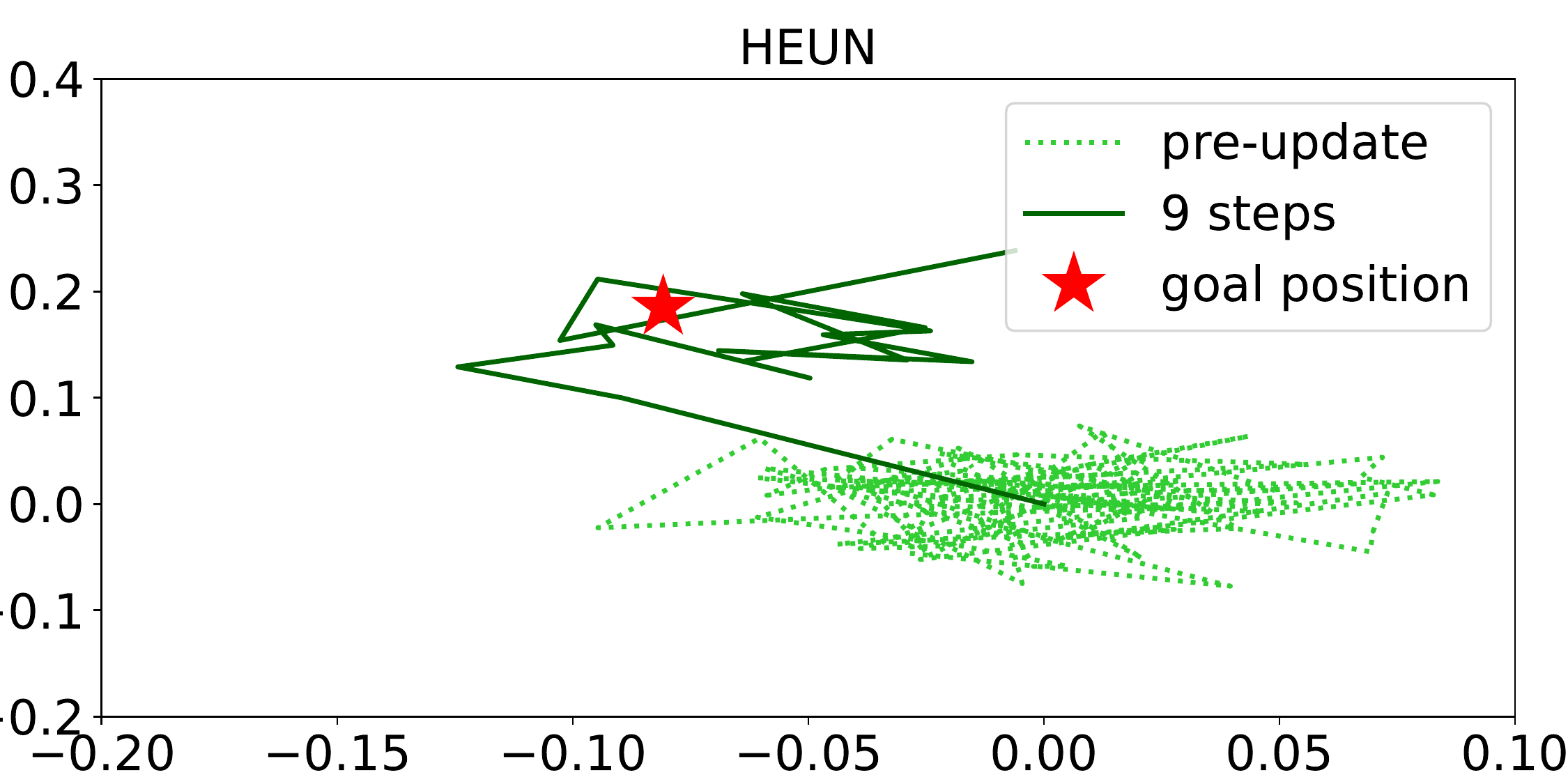} 
        \subcaption{Heun's}    
    \end{minipage}
    \begin{minipage}{0.244\linewidth}
        \includegraphics[width=\linewidth]{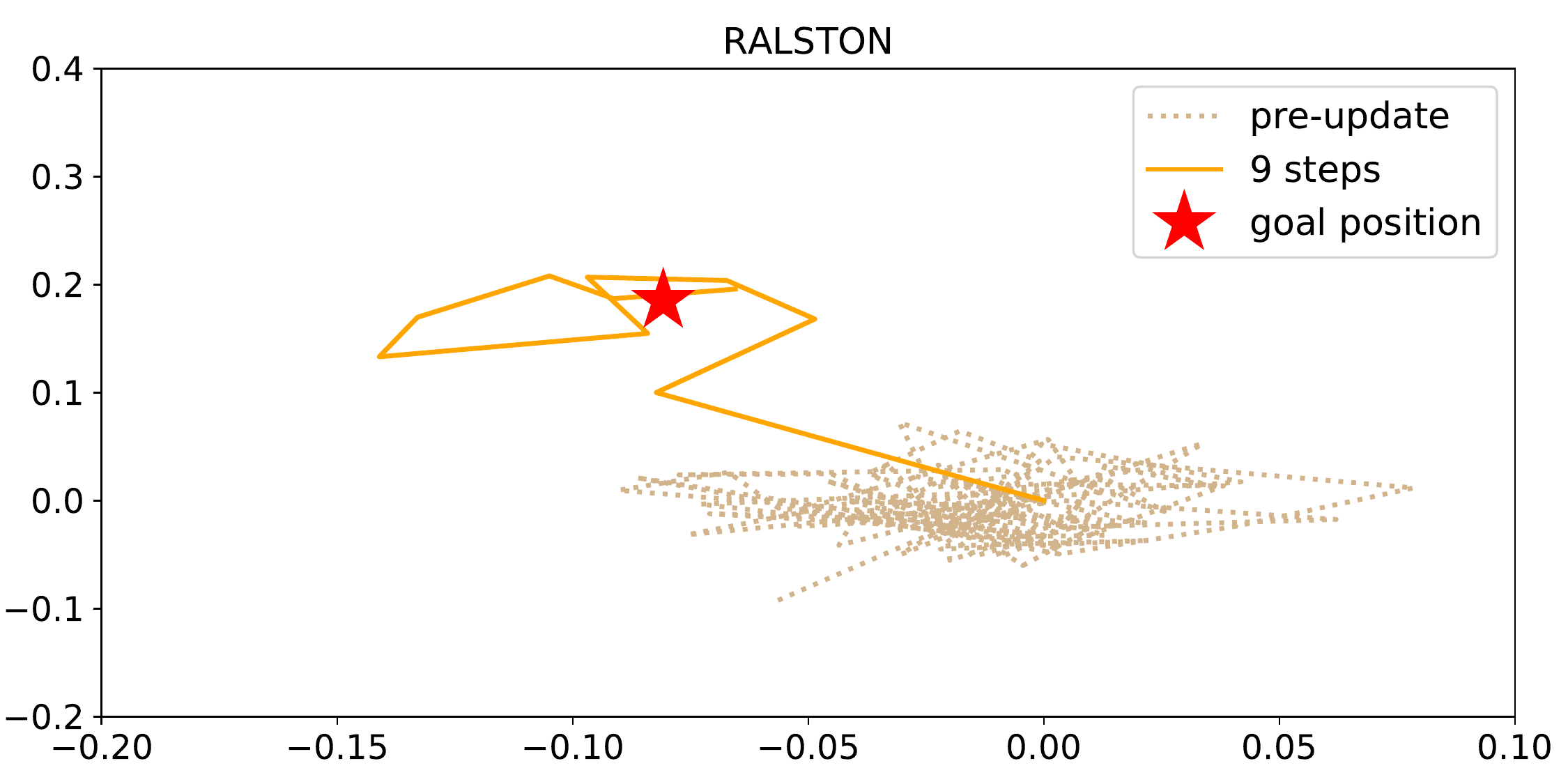} 
        \subcaption{Ralston}
    \end{minipage}
    \begin{minipage}{0.244\linewidth}
        \includegraphics[width=\linewidth]{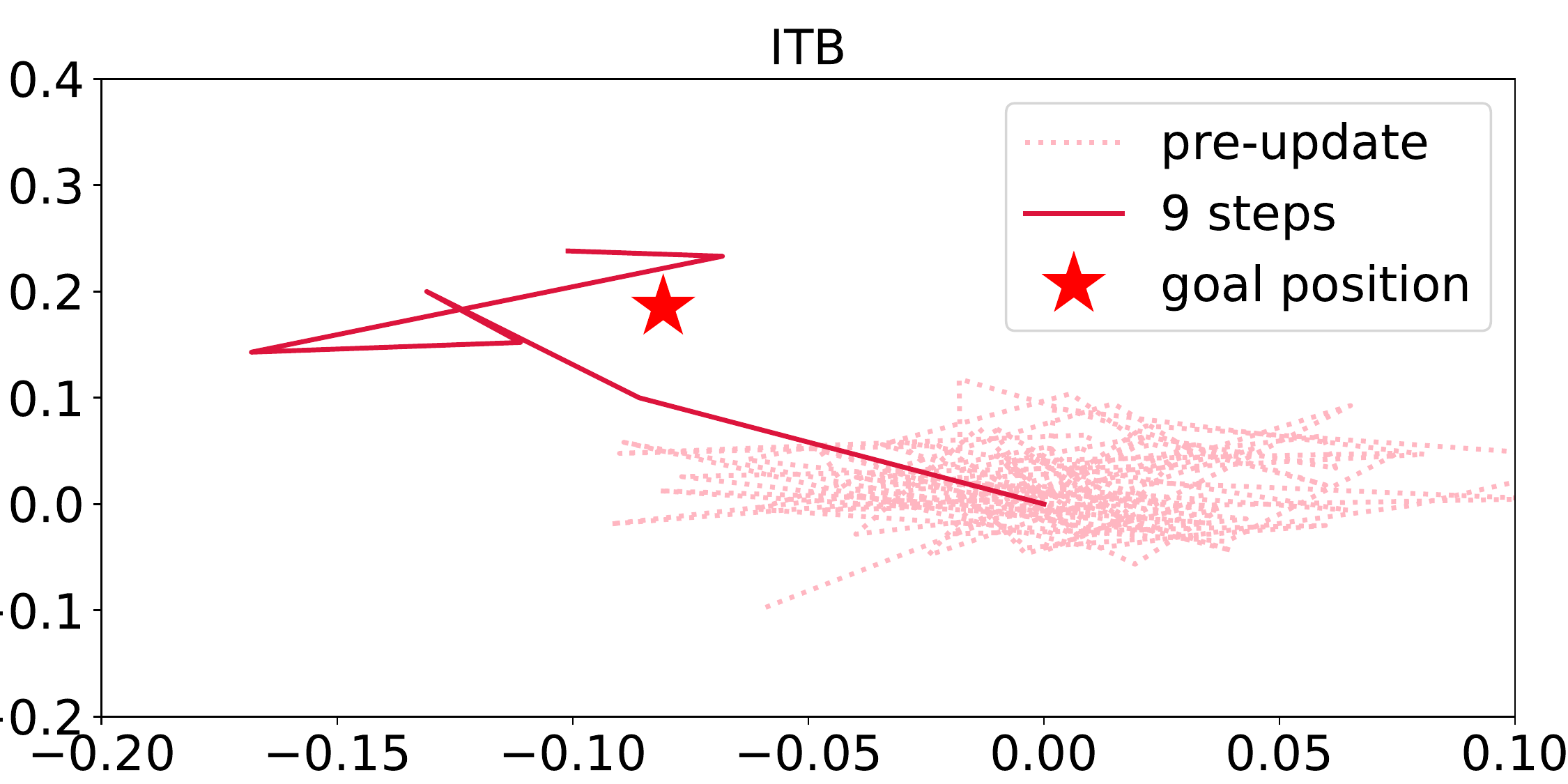} 
        \subcaption{ITB}    
    \end{minipage}
    \vspace{-0.15cm}
    \caption{Illustration of fine-tuning using MAML-RK2 versus MAML on 2D Navigation task}
    \label{fig:point_traj}
\end{figure*}

\begin{wrapfigure}{r}{0.5\textwidth}
    \centering
    \includegraphics[width=\linewidth]{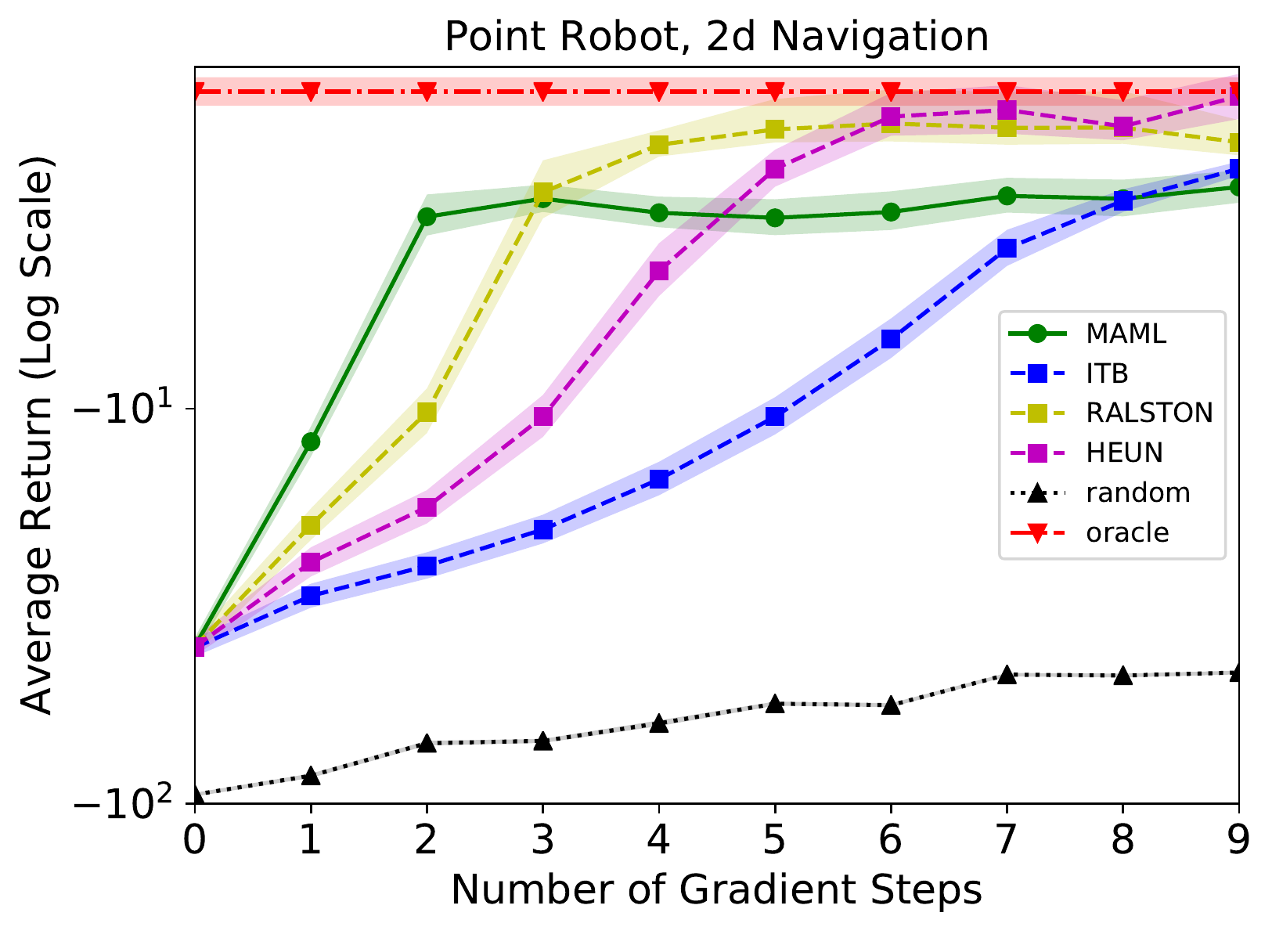} 
    \caption{The 2D navigation performance of MAML-RK2 over different number of gradient steps.}
    \label{fig:point_perf}
    \vspace{-0.4cm}
\end{wrapfigure}

\textbf{Reinforcement Learning} - We evaluate MAML-RK2 on two types of reinforcement
learning environments, 2D navigation and locomotion. For training, REINFORCE was used
for policy  \citep{Williams1992} and trust-region policy optimization (TRPO) was used for meta-optimization \citep{Schulman2015}.
For 2D navigation, there are a set of tasks where a point agent must 
reach to different goal locations in 2D space. The state, the action, and the reward 
corresponds to the 2D location, the motion velocity, and the distance to the goal 
respectively. The simulation terminates when an agent navigates within 0.01 distance 
from the goal.  

Figure~\ref{fig:point_perf} presents the performance of MAML-RK2 for up to nine
gradient updates with 40 samples. A model trained with random initialization
(black triangle line) performs poorly even with nine gradient updates. 
The red curve corresponds to the performance of model with oracle policy 
that receives the goal position as input.
Recall that $a_2$ corresponds to fast adaptation coefficient and $a_1$ corresponds to
good feature representation (because the higher the coefficient, the more it focuses on
$\nabla \mathcal{L}(\theta^\prime)$ and less on $\nabla \mathcal{L}(\theta)$). 
Interestingly, according to the plots in Figure \ref{fig:point_perf} the learning speed is ordered as follows: the midpoint, Ralston,
Heun's, and ITB methods.  This ordering corresponds to the ordering of coefficient of $a_2$,
which is $1$, $\frac{2}{3}$, $\frac{1}{2}$, and $\frac{1}{3}$ respectively.
The midpoint method suffers from poor performance the most. We suspect that
this is because it only emphasizes $\nabla \mathcal{L}(\theta^\prime)$ \change{(fast-adaptation)} part and
ignores $\nabla \mathcal{L}(\theta)$ \change{(shared-representation)} part. 
Ralston method seems to achieve the right balance between the two terms as the performance reaches to 
the oracle (red level) the fastest.

Lastly, Figure~\ref{fig:point_traj} illustrates the actual trajectory of learning towards
the final location. The results shows that the midpoint method takes long time to find the goal location and jitters a lot (showing suboptimal temporal dynamics). On the other hand, ITB finds the goal location
within very few steps. Although Ralston and Heun's are in between the midpoint and ITB,
both methods still take much fewer steps to converge to the goal. 
This clearly demonstrates the role between optimizing under 
 $\nabla \mathcal{L}(\theta^\prime)$ versus $\nabla \mathcal{L}(\theta)$.

\section{Discussion and future work}

In this paper, we extend the fast-adaptation stage (the inner loop) 
to higher-order Runge-Kutta methods in MAML to gain a finer control over the optimization, and show that original
fast-adaptation update corresponds to the second-order midpoint method.
The refined RK optimization helped us control various important aspects of the meta-learning process (fast adaptation and shared representation) 
%
%
%
achieving improved performance on 
regression, classification, and reinforcement learning tasks.

It is worth noting that our proposed generalization is not specific to MAML, but can also be applied to other
meta-models. We share some potential directions of future work.

{\bf Exploring other ODE integrators} -
We applied explicit Runge-Kutta ODE integrator to generalize 
stochastic gradient optimization of MAML. One can also explore 
other variations on gradient-based updates such as AdaGrad and ADAM \citep{Duchi2011, Kingma2015} and its effects to the meta-learning models. (Similar types of analysis has been done for image classification via neural networks, \citealp{Im2017loss}.)
%

Beside RK integrators, one can also apply other integrators, such as exponential- and leapfrog integrators. Since different integrators focus on different aspects of the optimization, one expects that they would benefit on different types of tasks.
We believe that
a thorough analysis of this would be an interesting direction to explore in the future and would be extremely beneficial to the practitioner.

{\bf Extension to ANIL} - 
\citet{Raghu2019} recently showed 
%
%
that feature reuse is the dominant factor in MAML optimization and propose 
%
to only train 
adaption updates on the last layer (i.e.\ the softmax layer for classification) of the 
model. 
It would be instructive to study how different RK methods effects the various layers of the MAML network.
For example, we can use the midpoint method for the last layer and apply Ralston, Heuns, ITB, and the gradient 
for the lower layers. This essentially has the effect of shifting the balance from $a_2$ to $a_1$ as we move down to the lower layers.

{\bf Extension to Bayesian MAML} - 
There are several works on Bayesian MAML \citep{Kim2018, Finn2018} that help in adding robustness and preventing overfitting to few shot learning. It would be interesting to combine MAML-RK optimization with these frameworks.

\clearpage
\bibliography{main}
\bibliographystyle{plain}
\clearpage
\appendix

\section{Experimental Details} 

Here, we include some of the experimental set-up details. 
Note that we worked based on original MAML code base. This makes the
the experimental set-up the same as the original paper.

\subsection{Regression}
We used multilayer perceptrons with two ReLU hidden layers of size 40.
We trained each model fore 100,000 iterations. Every gradient updates are using 10 batch samples.
The learning rate (step size) is set to 0.01.

\subsection{Classification}
\textbf{Omniglot dataset} 
Following the same model architectures from previous studies \citep{Vinyals2016, Finn2017}, 
we used four convolutonal blocks with 64 $3\times 3$ convolution filters, followed by 
batch normalization and  ReLU activations, and $2 \times 2$ max pooling. The last hidden
layer is 64.

We ran 30,000 epoch for 5-way 1-shot and 5-way 5-shot set-ups for all models.
We ran 50,000 epoch for 20-way 1-shot set-up for all models.
We ran 90,000 epoch for 20-way 5-shot set-up for all models.

Because there are total of $NK$ examples, where $N$ is number of tasks and $K$ is number of examples per task,
we were able to use 32 batch sizes for each gradient updates for meta-parameters. 
The network was evaluated using 3 gradient updates with step size 0.4 for 5-way set-up and 5 gradient updates with
step size 0.1 for 20-way set-up.

\textbf{MiniImagenet dataset} 
The same convolutional blocks are used as Omniglot dataset except that the number of filters were reduce to 32.
The batch size was set to 4 and 2 for meta-parameter updates. 
All models were trained for 90,000 epoch and 5 gradient updates with 0.01 learning rate during the training.
At test time, 10 gradient updates were used for 15 examples per class.

\subsection{Reinforcement learning (2D navigation) }
For all meta algorithms presented in this paper, the parameters are set to be the same. In the 2D navigation, like the original MAML paper, the model consists of a 2-layer MLP each with 100 units and it uses ReLU for activation. The fast adaptation learning is rate is 0.16 while the meta learning rate is 0.01. We trained it with 100 iterations and meta batch size of 40.

\end{document}